\documentclass[jair,twoside,11pt,theapa]{article}
\usepackage{jair, theapa, rawfonts}
\usepackage{latexsym}
\usepackage{amsmath}
\usepackage{amssymb}
\usepackage{amsthm}
\usepackage{verbatim}
\usepackage{graphicx}
\usepackage{grffile}
\usepackage{booktabs}
\usepackage{multirow}

\usepackage{url}

\DeclareMathOperator{\E}{E}
\DeclareMathOperator{\Cov}{Cov}
\DeclareMathOperator{\Var}{Var}
\DeclareMathOperator{\Tr}{Tr}
\DeclareMathOperator{\PMI}{PMI}
\DeclareMathOperator{\diag}{diag}

\newtheorem{theorem}{Theorem}
\newtheorem{lemma}{Lemma}
\newtheorem{corollary}{Corollary}[lemma]
\newtheorem{assum}{Assumption}
\theoremstyle{remark}
\newtheorem{remark}{Remark}[lemma]

\firstpageno{1}

\begin{document}

\title{Context Vectors are Reflections of Word Vectors in Half the Dimensions}

\author{\name Zhenisbek Assylbekov \email zhassylbekov@nu.edu.kz \\
\name Rustem Takhanov \email rustem.takhanov@nu.edu.kz \\
\addr Nazarbayev University, Department of Mathematics,\\
53 Kabanbay Batyr ave., Astana 010000 Kazakhstan}

% For research notes, remove the comment character in the line below.
% \researchnote

\maketitle

\begin{abstract}
This paper takes a step towards theoretical analysis of the relationship between word embeddings and context  embeddings in models such as word2vec. We start from basic probabilistic assumptions on the nature of word vectors, context vectors, and text generation. These assumptions are supported either empirically or theoretically by the existing literature. Next, we show that under these assumptions the widely-used word-word PMI matrix is approximately a random symmetric Gaussian ensemble. This, in turn, implies that context vectors are reflections of word vectors in approximately half the dimensions. As a direct application of our result, we suggest a theoretically grounded way of tying weights in the SGNS model.
\end{abstract}

\section{Introduction and Main Result}
Today word embeddings play an important role in many natural language processing tasks, from predictive language models and machine translation to image annotation and question answering, where they are usually `plugged in' to a larger model. An understanding of their properties is of interest as it may allow the development of better performing embeddings and improved interpretability of models using them. This paper takes a step in this direction. 

\noindent\textbf{Notation}: We let $\mathbb{R}$ denote the  real numbers. Bold-faced lowercase letters ($\mathbf{x}$) denote vectors in Euclidean space, bold-faced uppercase letters ($\mathbf{X}$) denote matrices, plain-faced lowercase letters ($x$) denote scalars, plain-faced uppercase letters ($X$) denote scalar random variables, $\|\cdot\|$ denotes the Euclidean norm: $\|\mathbf{x}\|:=\sqrt{\mathbf{x}^\top\mathbf{x}}$, `i.i.d.' stands for `independent and identically distributed'. We use the sign $\sim$ to abbreviate the phrase `distributed as', and the sign $\propto$ to abbreviate `proportional to'. $\Tr(\mathbf{A})$ is used to denote the trace of a matrix $\mathbf{A}$. $M_{\mathbf{x}}(\mathbf{t})$ is the moment-generating function of a random vector $\mathbf{x}$ at $\mathbf{t}$: $M_{\mathbf{x}}(\mathbf{t})=\E[e^{\mathbf{t}^\top\mathbf{x}}]$. $\odot$ is the Hadamard product (element-wise multiplication).

Assuming that words have already been converted into indices, let $\{1,\ldots,n\}$ be a finite vocabulary of words. Following the setup of the widely used {word2vec} model \shortcite{mikolov2013distributed}, we will use \textit{two} vectors per each word $i$: 
\begin{itemize}
    \item $\mathbf{w}_i$ when $i$ is a center word,
    \item $\mathbf{c}_i$ when $i$ is a context word.
\end{itemize}  
We make the following key assumptions in our work.
\begin{assum}\label{as:isotropy}
    A priori word vectors $\mathbf{w}_1,\ldots,\mathbf{w}_n\in\mathbb{R}^{d}$ are i.i.d. draws from isotropic multivariate Gaussian distribution:
    \begin{equation}
        \mathbf{w}_i\,\,{\stackrel{\text{iid}}{\sim}}\,\,\mathcal{N}\left(\mathbf{0},\,\textstyle{\frac1d}\mathbf{I}\right),\label{eq:wordvec}
    \end{equation}
    where $\mathbf{I}$ is the $d\times d$ identity matrix. 
\end{assum}
\noindent This is motivated by the work of \shortciteA{arora2016latent}, where the ensemble of word vectors consists of i.i.d draws generated by $\mathbf{v} = s\cdot\hat{\mathbf{v}}$, with $\hat{\mathbf{v}}$ being from the spherical Gaussian distribution $\mathcal{N}(\mathbf{0},\mathbf{I})$, and $s$ being a scalar random variable with bounded expectation and range. In their work, the norm $\|\mathbf{v}_i\|$ of the word vector for a word $i$ is related to its unigram probability $p(i)$, and to allow a sufficient dynamic range for these probabilities they needed the multiplier $s$. In our work, unigram probabilities are not mapped to vector lengths, and this is why we do not need such multiplier. Direct relationship between word probabilities and word vector norms is also implied by the model of \shortciteA{hashimoto2016word}.
\begin{assum}\label{as:context}
    Context vectors $\mathbf{c}_1,\ldots,\mathbf{c}_n$ are related to word vectors according to
    \begin{equation}
    \mathbf{c}_i=\mathbf{Qw}_i,\quad i=1,\ldots,n,\label{eq:cont_vec}
    \end{equation}
    for some orthogonal matrix $\mathbf{Q}\in\mathbb{R}^{d\times d}$.   
\end{assum}
\noindent This is mainly guided by the work of \shortciteA{press2017using}, who showed that context vectors in the SGNS model of \shortciteA{mikolov2013distributed} are distributed similarly to word vectors in the sense that pairwise cosine distances between word (input) embeddings strongly correlate with the corresponding pairwise cosine distances between context (output) embeddings (see their Table~4). This is why we choose the transform from word vectors to context vectors to be orthogonal as it preserves inner products and consequently Euclidean norms. Notice, that $\mathbf{c}_i\,\,{\stackrel{\text{iid}}{\sim}}\,\,\mathcal{N}\left(\mathbf{0},\frac1d\mathbf{I}\right).$

\begin{assum}\label{as:model}
Given a word $j$, probability of any word $i$ being in its context\footnote{Context is a fixed-size symmetric window around the given word.} is given by
\begin{equation}
p(i\mid j) \propto p_i\cdot{e^{\mathbf{w}_j^\top\mathbf{c}_i}}\label{eq:model}
\end{equation}
where $p_i=p(i)$ is the unigram probability for the word $i$, which is inverse proportional to its smoothed frequency rank $r_i$,  i.e. 
\begin{equation}
p_i\propto\frac{1}{r_i^{1-\alpha}}, \qquad\alpha\in(0, 1].\label{eq:unigram}
\end{equation}
\end{assum}
\noindent This is similar to the log-linear model of \shortciteA{arora2016latent}, but differs in the following aspects: $\mathbf{c}_i$ is not assumed to do a random walk over the unit sphere with bounded displacement; we use the factor $p_i$ to directly capture word frequencies and do not model them via vector norms. Equation~\eqref{eq:model} can be interpreted as follows: probability that the word $i$ occurs in the context of the word $j$ is probability that the word $i$ occurs anywhere in a large corpus, corrected for the relationship between words $i$ and $j$. This approach was already considered by \shortciteA{melamud2017simple} but in their work $i$ is the entire left context of the word $j$, and $\mathbf{c}_i$ is a vector representation of this entire context.  Also, like \shortciteA{arora2016latent} but unlike \shortciteA{melamud2017simple}, we use the model \eqref{eq:model} for a theoretical analysis rather than for fitting to data. Smoothing of the unigram probabilities (i.e. raising them to power $1-\alpha$) is motivated by the works of \shortciteA{mikolov2013distributed,levy2015improving,pennington2014glove}, where $\alpha=0.25$ is a typical choice. We notice here that $\alpha=0^+$ gives us Zipf's law \cite{zipf1935psycho}, whereas $\alpha=1$ gives us uniform distribution of word frequencies which is not valid empirically but on the other hand can be used to explain additivity of word vectors \cite{gittens2017skip}. The specific value of $\alpha$ is important for 

The relationship between word (input) and context (output) vectors was addressed in several previous works. E.g., in recurrent neural network language modeling (RNNLM), tying input and output embeddings is a useful regularization technique introduced earlier \shortcite{bengio2001neural} and studied in more details recently \shortcite{press2017using,inan2016tying}. This technique improves language modeling quality (measured as perplexity of a held-out text) while decreasing the total number of trainable parameters almost two-fold since most of the parameters in RNNLM are due to embedding matrices. The direct application of this regularization technique to SGNS worsens the quality of word vectors as was shown empirically by \shortciteA{press2017using} and by \shortciteA{gulordava2018represent}. This worsening was predicted earlier by \shortciteA{goldberg2014word2vec} using a simple linguistic observation that words usually do not appear in the contexts of themselves. This basically means that $\mathbf{Q}\ne\mathbf{I}$ in \eqref{eq:cont_vec}. At the same time, there is empirical evidence that the relationship between input and output embeddings is \textit{linear} \shortcite{mimno2017strange,gulordava2018represent}. In this paper, we provide a theoretical justification for this and reveal the exact form of the transform $\mathbf{Q}$. Our main contribution is the following 
\begin{theorem}\label{th:main}
Under Assumptions \ref{as:isotropy}, \ref{as:context}, and \ref{as:model} above, context vector $\mathbf{c}_i$ for a word $i$ is the reflection of the word vector $\mathbf{w}_i$ in approximately half of the dimensions.
\end{theorem}
\begin{figure}
    \centering
    \includegraphics[width=.5\textwidth]{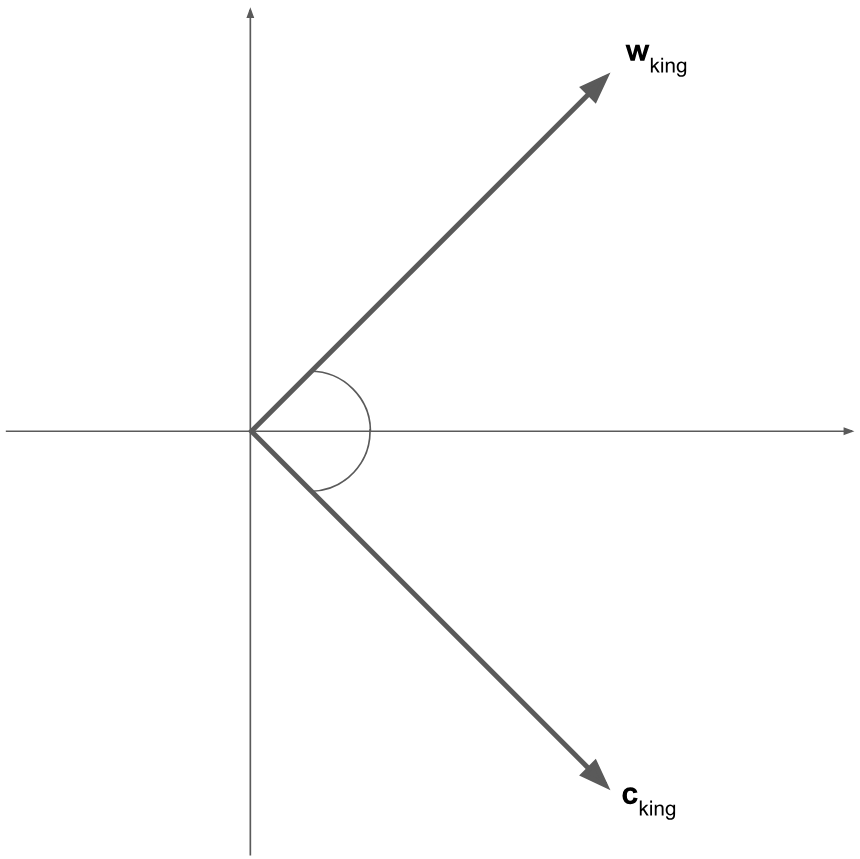}
    \caption{Context vector is a reflection of word vector in half the coordinates.}
    \label{fig:wc_vec}
\end{figure}
\noindent Figure \ref{fig:wc_vec} illustrates this idea for the case $d=2$. In general, our word and context vectors live in a $d$-dimensional vector space over real numbers ($\mathbb{R}^d$). By Theorem \ref{th:main} we can settle them in a $d/2$-dimensional vector space over complex numbers ($\mathbb{C}^{d/2}$) in such way that the context vector $\widetilde{\mathbf{c}}_i\in\mathbb{C}^{d/2}$ for a word $i$ is the \textit{complex conjugate} of the word vector $\widetilde{\mathbf{w}}_i\in\mathbb{C}^{d/2}$. This is in line with the Theorem 2 of \shortciteA{allen2018vec}, however they use completely different set of basic assumptions and their primary goal is to encode statistical properties of words directly into word vectors.

\section{Proof of Theorem \ref{th:main}}
The proof is divided into three steps: first we show that the partition function in \eqref{eq:model} concentrates around $1$, and thus $\propto$ can be replaced by $\approx$; using this fact we show that $\mathbf{Q}$ is (approximately) an involutary matrix, i.e. similar to $\diag(\{+1, -1\})$; and finally we show that the word-word pointwise-mutual information matrix is approximately symmetric Gaussian random matrix with weakly dependent entries. The latter fact immediately implies the statement of the Theorem \ref{th:main}.
\subsection{Concentration of the partition function}
We first need the following auxiliary result.
\begin{lemma}\label{lem:concentr_wqw}
Let $\mathbf{w}\sim\mathcal{N}(0,\sigma^2\mathbf{I})$, then $\forall t>0$ and any orthogonal $\mathbf{Q}$
\begin{equation}
\Pr\left(\left|\mathbf{w}^\top\mathbf{Qw}-\Tr(\mathbf{Q})\sigma^2\right|>t\sqrt{2d}\sigma^2\right)
\le\frac{1}{t^2}.\label{eq:concentr_wqw}
\end{equation}
\end{lemma}
\begin{proof}
Consider the random variable $X = {\mathbf w}^\top\mathbf{Q} {\mathbf w}$, and let $\mathbf{Q} = [q_{ij}]$. We have
\begin{equation}
\E[X] = \E\left[\sum_{i,\,j} q_{ij} w_i w_j\right] = \sum_{i} q_{ii} \sigma^2=\Tr(\mathbf{Q})\sigma^2\label{eq:first_moment}    
\end{equation}
Therefore, 
\begin{equation}
(\E[X])^2 = \sum_{i,\,j} q_{ii} q_{jj} \sigma^4\label{eq:first_moment_sq}
\end{equation}
Further,
$$
\E\left[X^2\right] = \E\left[\left(\sum_{i,\,j} q_{ij} w_i w_j\right)^2\right] =\E\left[\sum_{i,\,j} q_{ii} q_{jj} w_i^2 w_j^2 + \sum_{i\ne j} (q^2_{ij} + q_{ij}q_{ji}) w_i^2 w_j^2\right]
$$
where we dropped the terms containing odd powers of $w_i$, as their expectations are zeros. Hence,
\begin{align}
\E\left[X^2\right] &= \sum_{i\ne j} q_{ii}q_{jj} \sigma^4 + \sum_{i}q_{ii}^2 \E[w_i^4] + \sum_{i\ne j}(q^2_{ij} + q_{ij}q_{ji})\sigma^4\notag\\
&=\sum_{i\ne j} q_{ii} q_{jj} \sigma^4 + 3\sum_{i}q_{ii}^2\sigma^4 + \sum_{i\ne j}(q^2_{ij} + q_{ij}q_{ji})\sigma^4 \notag\\ &= \sum_{i,\,j}q_{ii}q_{jj}\sigma^4 + 2\sum_{i}q_{ii}^2\sigma^4 + \sum_{i\ne j}(q^2_{ij} + q_{ij}q_{ji})\sigma^4\label{eq:second_moment}
\end{align}
From \eqref{eq:first_moment_sq} and \eqref{eq:second_moment} we have
$$
\Var[X] = \E\left[X^2\right] - (\E[X])^2 
= 2\sum_{i} q_{ii}^2 \sigma^4 + \sum_{i\ne j} (q^2_{ij} + q_{ij} q_{ji})\sigma^4    
$$
It is easy to see that $\sum_{i}q_{ii}^2 + \sum_{i\ne j} q^2_{ij} = d$ (the sum of squared elements of an orthogonal matrix). In this way,
\begin{equation}
\Var[X] = \left(d + \sum_{i} q_{ii}^2 + \sum_{i\ne j} q_{ij} q_{ji}\right) \sigma^4 = \left(d + \Tr(\mathbf{Q}^2)\right)\sigma^4\label{eq:variance}
\end{equation}
Applying Chebyshev inequality to $X$, and taking into account \eqref{eq:first_moment} and \eqref{eq:second_moment}, we have $\forall\epsilon>0$
$$
\Pr\left(\left|\mathbf{w}^\top\mathbf{Qw}-\Tr(\mathbf{Q})\sigma^2\right|>t\sqrt{2d}\sigma^2\right)
\le\frac{\left(d+\Tr(\mathbf{Q}^2)\right)\sigma^4}{t^2\cdot2d\sigma^4}    
$$
Since $\mathbf{Q}^2$ is orthogonal, its trace does not exceed $d$, and we obtain \eqref{eq:concentr_wqw}.
\end{proof}

\begin{remark}\label{rem:wqw_wordvecs}
When $\sigma^2=\frac1d$, the inequality~\eqref{eq:concentr_wqw} becomes:
\begin{equation}
\Pr\left(\left|\mathbf{w}^\top\mathbf{Qw}-\frac{\Tr(\mathbf{Q})}{d}\right|>t\sqrt{\frac2d}\right)
\le\frac{1}{t^2}.\label{eq:concentr_wqw2}
\end{equation}
\end{remark}

\begin{corollary}\label{col:concentr_wnorm}
Let $\mathbf{w}\sim\mathcal{N}(0,\frac1d\mathbf{I})$, then $\forall t>0$
\begin{equation}
\Pr\left(\left|\|\mathbf{w}\|^2-1\right|>t\sqrt{\frac2d}\right)\le\frac{1}{t^2}.\label{eq:w_norm_wordvecs}
\end{equation}
\end{corollary}
\begin{proof}
The statement follows from Lemma~\ref{lem:concentr_wqw} and its Remark~\ref{rem:wqw_wordvecs}, when $\mathbf{Q}=\mathbf{I}$.
\end{proof}

\noindent Now we are ready to show that the partition function in \eqref{eq:model} concentrates around $1+\frac{1}{2d}$.
\begin{lemma}\label{lem:concentr_Z}
Let $Z_j$ be a partition function in \eqref{eq:model}, i.e. $Z_j=\sum_{i=1}^n p_i e^{\mathbf{w}_j^\top\mathbf{c}_i}$. Then
\begin{equation}
Z_j\approx 1+\frac{1}{2d}\quad(\forall j)\label{eq:concentr_Z}
\end{equation}
\end{lemma}
\begin{proof}
We will first show that the conditional expectation $\E[Z_j\mid\mathbf{w}_j]$ depends on $\mathbf{w}_j$ mainly through its norm $\|\mathbf{w}_j\|$:
\begin{align}
\E\left[Z_j\mid\mathbf{w}_j\right]&=\E\left[\sum_{i=1}^n p_i e^{\mathbf{w}_j^\top\mathbf{c}_i}\bigm|\mathbf{w}_j\right]
=p_j\E\left[e^{\mathbf{w}_j^\top\mathbf{c}_j}\bigm|\mathbf{w}_j\right]+\sum_{i\ne j}p_i\E\left[e^{\mathbf{w}_j^\top\mathbf{c}_i}\bigm|\mathbf{w}_j\right]\notag\\
&=p_j\E\left[e^{\mathbf{w}_j^\top\mathbf{Qw}_j}\bigm|\mathbf{w}_j\right]+\sum_{i\ne j}p_i M_{\mathbf{c}_i}(\mathbf{w}_j)=p_je^{\mathbf{w}_j^\top\mathbf{Qw}_j}+\sum_{i\ne j}p_ie^{\frac12\mathbf{w}_j^\top(\frac1d\mathbf{I})\mathbf{w}_j}\notag\\
&=p_je^{\mathbf{w}_j^\top\mathbf{Qw}_j}+e^{\frac{1}{2d}\|\mathbf{w}_j\|^2}\sum_{i\ne j}p_i=p_j\left(e^{\mathbf{w}_j^\top\mathbf{Qw}_j}-e^{\frac{1}{2d}\|\mathbf{w}_j\|^2}\right)+e^{\frac{1}{2d}\|\mathbf{w}_j\|^2}\label{eq:mean_z}
\end{align}
where $M_{\mathbf{c}_i}(\mathbf{w}_j)$ is the moment-generating function of $\mathbf{c}_i$ at $\mathbf{w}_j$. In Lemma~\ref{lem:concentr_wqw} and Corollary~\ref{col:concentr_wnorm} it is shown that $\mathbf{w}_j^\top\mathbf{Qw}_j$ and  $\|\mathbf{w}_j\|^2$ concentrate well around their means $\Tr(\mathbf{Q})\frac1d$ and $1$ respectively and thus we can approximate
\begin{align}
\mathbf{w}_j^\top\mathbf{Qw}_j&\approx\frac{\Tr(\mathbf{Q})}d,\label{eq:wqw_approx}\\
\|\mathbf{w}_j\|^2&\approx 1.\label{eq:norm_approx}
\end{align}
The quantity $\frac{1}{2d}$ is small for $d\ge50$ which is typical for dimensionality of word vectors \shortcite{mikolov2013distributed}. Thus, using \eqref{eq:wqw_approx}, \eqref{eq:norm_approx}, and Maclaurin expansion for $x\mapsto e^x$ in the last term of \eqref{eq:mean_z}, we obtain
\begin{equation}
\E\left[Z_j\mid\mathbf{w}_j\right]\approx p_j\left(e^{\frac{\Tr(\mathbf{Q})}d}-e^{\frac{1}{2d}}\right)+1+\frac{1}{2d}.\label{eq:mean_z_approx}
\end{equation}
This approximation is very helpful as the right-hand side does not contain $\mathbf{w}_j$ and thus it is an approximation for the $\E[Z_j]$ as well. Let $H_{n,\alpha}$ be the normalizer in  \eqref{eq:unigram}, then 
$$
H_{n,\alpha}=\sum_{k=1}^n\frac{1}{k^{1-\alpha}}\sim\int_{1}^n\frac{dx}{x^{1-\alpha}}\sim \frac{n^{\alpha}}{\alpha},
$$ 
and thus we have\footnote{We abuse the notation here and use `$\sim$' to denote asymptotic equivalence, i.e. $f(n)\sim g(n)$ iff $\lim_{n\to\infty}\frac{f(n)}{g(n)}=1$.} 
\begin{equation}
p_j\sim\frac{\alpha}{r_j^{1-\alpha}n^\alpha}\to0\text{ as }n\to\infty.\label{eq:pj}
\end{equation}
Now, combining \eqref{eq:mean_z_approx}, \eqref{eq:pj}, and \eqref{eq:norm_approx}, we get
\begin{equation}
\E\left[Z_j\mid\mathbf{w}_j\right]\approx1+\frac{1}{2d}.\label{eq:mean_z_approx2}
\end{equation}
Now let us show that $\Var[Z_j]$ is small relative to the mean $\E[Z_j]$. First, we have
\begin{equation}
\Var\left[Z_j\mid\mathbf{w}_j\right]=\sum_{i=1}^np_i^2\Var\left[e^{\mathbf{w}_j^\top\mathbf{c}_i}\bigm|\mathbf{w}_j\right]+\sum_{i\ne k}p_i p_k\Cov\left[e^{\mathbf{w}_j^\top\mathbf{c}_i},\,e^{\mathbf{w}_j^\top\mathbf{c}_k}\bigm|\mathbf{w}_j\right]\label{eq:var_first}
\end{equation}
For the variance terms we have
\begin{equation}
\Var\left[e^{\mathbf{w}_j^\top\mathbf{c}_j}\bigm|\mathbf{w}_j\right]=\Var\left[e^{\mathbf{w}_j^\top\mathbf{Qw}_j}\bigm|\mathbf{w}_j\right]=0,\label{eq:var_zero}
\end{equation}
and
\begin{align}
\Var\left[e^{\mathbf{w}_j^\top\mathbf{c}_i}\bigm|\mathbf{w}_j\right]&=\E\left[e^{2\mathbf{w}_j^\top\mathbf{c}_i}\bigm|\mathbf{w}_j\right]-\left(\E\left[e^{\mathbf{w}_j^\top\mathbf{c}_i}\bigm|\mathbf{w}_j\right]\right)^2\notag\\
&=M_{\mathbf{c}_i}(2\mathbf{w}_j)-\left(M_{\mathbf{c}_i}(\mathbf{w}_j)\right)^2\notag\\
&=e^{\frac2d\|\mathbf{w}_j\|^2}-e^{\frac1d\|\mathbf{w}_j\|^2}.\label{eq:var_nnz}
\end{align}
Conditioned on $\mathbf{w}_j$, the random variables $\{e^{\mathbf{w}_j^\top\mathbf{c}_i}\}_{i\ne j}$ are independent, while $e^{\mathbf{w}_j^\top\mathbf{c}_j}$ is constant, and thus
\begin{equation}
\Cov\left[e^{\mathbf{w}_j^\top\mathbf{c}_i},\,e^{\mathbf{w}_j^\top\mathbf{c}_k}\bigm|\mathbf{w}_j\right]=0,\quad i\ne k.\label{eq:cov_zero}
\end{equation}
From \eqref{eq:var_first}, \eqref{eq:var_zero}, \eqref{eq:var_nnz}, and \eqref{eq:cov_zero}, we have
\begin{equation}
\Var[Z_j\mid\mathbf{w}_j]=\left(e^{\frac2d\|\mathbf{w}_j\|^2}-e^{\frac1d\|\mathbf{w}_j\|^2}\right)\sum_{i\ne j}p_i^2.\label{eq:var_simpl} 
\end{equation}
Notice that for some constant $C>0$
\begin{equation}
\sum_{i\ne j}p_i^2\sim\frac{C}{n^{2\alpha}}\int_1^n\frac{dx}{x^{2-2\alpha}}\sim\frac{C}{n^{\min\{1,\,2\alpha\}}}.\label{eq:pi_sq}
\end{equation}
Combining \eqref{eq:norm_approx}, \eqref{eq:var_simpl}, \eqref{eq:pi_sq}, and using Maclaurin expansion for $x\mapsto e^x$, we have
\begin{equation}
\Var[Z_j\mid\mathbf{w}_j]\sim\frac{C}{dn^{\min\{1,\,2\alpha\}}}\to0\quad\text{as}\quad n\to\infty.\label{eq:var_to_zero}    
\end{equation}
Now, the statement of the lemma follows from \eqref{eq:mean_z_approx2} and \eqref{eq:var_to_zero}.
\end{proof}
\begin{remark}
Lemma \ref{lem:concentr_Z} basically says that under the Assumptions \ref{as:isotropy} and \ref{as:context}, the model \eqref{eq:model} self-normalizes, i.e. the normalization term is almost constant and moreover it is almost $1$. This result is similar to the result of \shortciteA{andreas2015and}, but differs in that our model \eqref{eq:model} is not log-linear as its condition ($j$) and prediction ($i$) are both parameterized. The result of \shortciteA{goldberger2018self} on self-normalization of the NCE language models is closer to ours but the setup differs in that $p_i$ does not appear as a factor in their model. We finally notice that Lemma \ref{lem:concentr_Z} is an analogue of Lemma 2.1 from \shortciteA{arora2016latent} but adapted to our settings.
\end{remark}

\subsection{$\mathbf{Q}$ is an involutary matrix}
\begin{lemma}
Let $\mathbf{Q}$ be the matrix mapping word vectors to context vectors as in \eqref{eq:cont_vec}. Then, under Assumptions \ref{as:isotropy}, \ref{as:context}, and \ref{as:model}, $\mathbf{Q}$ is approximately an involutary matrix.
\end{lemma}
\begin{proof}
\noindent The dimensionality $d$ of word vectors is usually in the range $[50, 1000]$ \shortcite{mikolov2013distributed}, and thus we can neglect the term $\frac{1}{2d}$ in \eqref{eq:concentr_Z} and approximate $Z_j\approx1$. This means that the model \eqref{eq:model} simplifies to
$$
p(i\mid j)\approx p(i)\cdot e^{\mathbf{w}_j^\top\mathbf{c}_i},
$$
or, equivalently
\begin{equation}
\ln\frac{p(i, j)}{p(i)p(j)}\approx\mathbf{c}_i^\top\mathbf{w}_j.\label{eq:pmi_factorization}
\end{equation}
where $p(i,\,j)$ is the probability that the words $i$ and $j$ co-occur in the same context window. Notice that the left-hand side in \eqref{eq:pmi_factorization} is the pointwise mutual information (PMI) between words $i$ and $j$. From \eqref{eq:cont_vec}
and \eqref{eq:pmi_factorization} we have
$$
\PMI(i, j)\approx\mathbf{w}_i^\top\mathbf{Q}^\top\mathbf{w}_j\quad\Leftrightarrow\quad \PMI\approx\mathbf{WQ}^\top\mathbf{W}^\top    
$$
where $\PMI$ stands for the PMI-matrix, and  $\mathbf{W}$ is a $n\times d$ matrix in which $i$-th row is $\mathbf{w}_i^\top$.
Since $p(i, j)=p(j, i)$, we should have $\PMI=\PMI^\top$, which implies
\begin{align}
\mathbf{W}\mathbf{Q}^\top\mathbf{W}^\top&\approx\mathbf{W}\mathbf{Q}\mathbf{W}^\top\notag\\
\Leftrightarrow \mathbf{W}^\top\mathbf{W}\mathbf{Q}^\top\mathbf{W}^\top\mathbf{W}&\approx\mathbf{W}^\top\mathbf{W}\mathbf{Q}\mathbf{W}^\top\mathbf{W}\notag\\
\Leftrightarrow \mathbf{Q}^\top&\approx\mathbf{Q},\label{eq:q_sym}
\end{align}
where we used the fact that $\mathbf{W}^\top\mathbf{W}\approx\frac1d\mathbf{I}$. Since $\mathbf{Q}$ is assumed to be orthogonal, from \eqref{eq:q_sym} we get
$$
\mathbf{Q}^2\approx\mathbf{I}
$$
Thus, $\mathbf{Q}$ is approximately an involutary matrix, and we can choose it to be a signature matrix, i.e. a diagonal matrix with $\pm1$ on the diagonal\footnote{One can show that any involutary matrix can be represented as $\mathbf{P}^\top\diag(\pm1,\ldots,\pm1)\mathbf{P}$, where $\mathbf{P}$ is orthogonal, and thus by reparametrization $\widetilde{\mathbf{w}}_i=\mathbf{P}\mathbf{w}_i\,\,{\stackrel{\text{iid}}{\sim}}\,\,\mathcal{N}(\mathbf{0},\frac1d\mathbf{I})$, we can still have \eqref{eq:q_diag}.}:
\begin{equation}
\mathbf{Q}:=\diag(\pm1,\ldots,\pm1).\label{eq:q_diag}
\end{equation}
\end{proof}
In this way, context vectors are word vectors with some of the coordinates being multiplied by $-1$. The natural question is: how many of the coordinates should be ``flipped''?

\subsection{PMI as a random matrix}
Let $\mathbf{x}_i\in\mathbb{R}^l$ be the vector consisting of the first $l$ coordinates of $\mathbf{w}_i$, i.e.
\begin{equation}
\mathbf{x}_i=\mathbf{w}_{i,1:l}=\begin{bmatrix}w_{i1},\ldots,w_{il}\end{bmatrix},\label{eq:x_vec}
\end{equation}
and let $\mathbf{y}_i\in\mathbb{R}^{d-l}$ be the vector consisting of the last $d-l$ coordinates of $\mathbf{w}_i$, i.e.
\begin{equation}
\mathbf{y}_i=\mathbf{w}_{i,\,l+1:d}=\begin{bmatrix}w_{i(l+1)},\ldots,w_{id}\end{bmatrix}.\label{eq:y_vec}
\end{equation}
Due to Assumption \eqref{eq:wordvec}, $\mathbf{x}_i$'s are i.i.d. draws from $\mathcal{N}\left(0,\frac1d\mathbf{I}_{l\times l}\right)$, $\mathbf{y}_i$'s are i.i.d. draws from $\mathcal{N}\left(0, \frac1d\mathbf{I}_{(d-l)\times(d-l)}\right)$, and $\{\mathbf{x}_i\}$, $\{\mathbf{y}_i\}$ are jointly independent.
Without restricting the generality, assume that the first $l$ diagonal elements in \eqref{eq:q_diag} are equal to $+1$, and the rest $d-l$ elements are equal to $-1$.
Thus $$
\mathbf{w}_i^\top\mathbf{Q}^\top\mathbf{w}_j=\mathbf{x}_i^\top\mathbf{x}_j-\mathbf{y}_i^\top\mathbf{y}_j
$$
For $i\ne j$, $\mathbf{x}_i^\top\mathbf{x}_j$ is a sum of  $l$ i.i.d. random variables with mean 0 and variance $\frac{1}{d^2}$, and by Central Limit Theorem, $\mathbf{x}_i^\top\mathbf{x}_j\approx\mathcal{N}\left(0, \frac{l}{d^2}\right)$. Similarly, $\mathbf{y}_i^\top\mathbf{y}_j\approx\mathcal{N}\left(0,\frac{d-l}{d^2}\right)$, and thus
\begin{equation}
\mathbf{w}_i^\top\mathbf{Q}^\top\mathbf{w}_j\approx\mathcal{N}\left(0, \frac{1}{d}\right),\quad i\ne j.\label{eq:pmi_nondiag}
\end{equation}
For $i=j$, we have
\begin{equation}
    \mathbf{w}_i^\top\mathbf{Q}^\top\mathbf{w}_i=\|\mathbf{x}_i\|^2-\|\mathbf{y}_i\|^2\\
    \sim\frac1d(\chi^2_{l}-\chi^2_{d-l})\approx\mathcal{N}\left(\frac{2l-d}{d}, \frac2d\right),\label{eq:pmi_diag}
\end{equation}
where $\chi^2_l$ is a chi-square random variable with $l$ degrees of freedom. 
By combinatorial argument (similar to that of Lemma~\ref{lem:concentr_wqw}) one can show that covariance between any two distinct and non-symmetric entries of $\mathbf{WQ}^\top\mathbf{W}^\top$ is zero, and thus
\begin{align}
&\Cov\left[\PMI_{ij},\,\PMI_{pq}\right]\approx0\notag\\&\forall\,(i,j)\ne(p,q),\,(p,q)\ne(j,i)\label{eq:zero_cov}
\end{align}
Moreover, we can show that $\PMI_{ij}$ and $\PMI_{pq}$ tend to be independent when $d$ is large enough. For the case $i\ne p$, $j\ne q$ this follows directly from the Assumption \eqref{eq:wordvec}. Now consider the case $i=p$, $j\ne q$ (two distinct elements from the same row). The case $i\ne p$, $j=q$ (two distinct elements from the same column) can be analyzed similarly. Let $\mathbf{t}=\begin{bmatrix}t_1 & t_2\end{bmatrix}^\top$. Then the moment-generating function (m.g.f.) of $\begin{bmatrix}\PMI_{ij} & \PMI_{iq}\end{bmatrix}^\top$ at $\mathbf{t}$ is
\begin{align}
M_{[\PMI_{ij}\,\, \PMI_{iq}]^\top}(\mathbf{t})&=\E\left[e^{t_1\mathbf{w}_i^\top\mathbf{c}_j+t_2\mathbf{w}_i^\top\mathbf{c}_q}\right]
=\E\left[\E\left[e^{t_1\mathbf{w}_i^\top\mathbf{c}_j+t_2\mathbf{w}_i^\top\mathbf{c}_j}\bigm|\mathbf{w}_i\right]\right]\notag\\
&=\E\left[M_{\mathbf{c}_j}(t_1\mathbf{w}_i)\cdot M_{\mathbf{c}_q}(t_2\mathbf{w}_i)\right]
=\E\left[e^{\frac{1}{2d} t_1^2\|\mathbf{w}_i\|^2}\cdot e^{\frac{1}{2d} t_2^2\|\mathbf{w}_i\|^2}\right]\notag\\
&=\E\left[e^{\frac{1}{2d}\|\mathbf{t}\|^2\|\mathbf{w}_i\|^2}\right]=M_{\chi^2_d}\left(\frac{1}{2d^2}\|\mathbf{t}\|^2\right)
=\left(1-\frac{\|\mathbf{t}\|^2}{d^2}\right)^{-\frac{d}{2}}\notag\\&=e^{\frac{\|\mathbf{t}\|^2}{2d}}\left[1+O\left(\frac{1}{d^3}\right)\right],%\left[\left(1-\frac{\|\mathbf{t}\|^2}{d^2}\right)^{-\frac{d^2}{\|\mathbf{t}\|^2}}\right]^{\frac{\|\mathbf{t}\|^2}{2d}},\notag
\end{align}
and the last expression for large $d$ is approximately $e^{\frac{\|\mathbf{t}\|^2}{2d}}$ which is the m.g.f. of a two-dimensional Gaussian vector with distribution $\mathcal{N}(\mathbf{0},\frac1d\mathbf{I}_{2\times2})$. Hence 
\begin{equation}\begin{bmatrix}\PMI_{ij} & \PMI_{iq}\end{bmatrix}^\top\approx\mathcal{N}\left(\mathbf{0},\textstyle{\frac1d}\mathbf{I}_{2\times2}\right),\label{eq:joint_gaus}
\end{equation}
which implies approximate independence between $\PMI_{ij}$ and $\PMI_{iq}$.
Hence, from \eqref{eq:pmi_nondiag} and \eqref{eq:pmi_diag} we conclude that for the PMI matrix
\begin{itemize}
    \item the  above-diagonal entries have (approximately) distribution $\mathcal{N}\left(0, \frac1d\right)$, 
    \item the diagonal entries have (approximately) distribution $\mathcal{N}\left(\frac{2l-d}{d}, \frac2d\right)$,
    \item all entries on and above its diagonal tend to pairwise independence.
\end{itemize}
This means that the PMI matrix is an approximately symmetric Gaussian random matrix with weakly dependent entries and it is known that the empirical  distribution of eigenvalues of such matrix approaches \textit{symmetric around $0$} distribution as its size $n$ increases \shortcite{de1999norm}.\footnote{Usually papers on random matrix theory do not state this explicitly. Rather they show that the limiting distribution has odd moments equal to 0. This is the case for the referenced paper as well, see the proof of their Theorem 2.2.} Thus, we should have
\begin{equation}
    \Tr(\PMI)\approx0\quad\Leftrightarrow\quad\sum_{i=1}^n\mathbf{w}_i^\top\mathbf{Q}^\top\mathbf{w}_i\approx0\quad
    \Leftrightarrow\quad \sum_{i=1}^n\|\mathbf{x}_i\|^2\approx\sum_{i=1}^n\|\mathbf{y}\|^2.\label{eq:tr_zero}
\end{equation}
Recall, that $d\|\mathbf{x}_i\|^2\sim\chi^2_l$ and $d\|\mathbf{y}_i\|^2\sim\chi^2_{d-l}$. Hence, taking expectation on both sides of \eqref{eq:tr_zero} we have
\begin{equation}
    \frac{n}{d}\E\left[\chi^2_l\right]\approx \frac{n}{d}\E\left[\chi^2_{d-l}\right]\quad \Leftrightarrow\quad l\approx d-l\quad\Leftrightarrow\quad l\approx\frac{d}{2}.\label{eq:main}
\end{equation}
which concludes the proof of Theorem~\ref{th:main}.
\begin{remark}
In terms of the introduced notation \eqref{eq:x_vec} and \eqref{eq:y_vec}, each word's vector $\mathbf{w}_i$ splits into two subvectors $\mathbf{x}_i$ and $\mathbf{y}_i$, and due to Theorem \ref{th:main}, our model \eqref{eq:model} for generating a word $i$ in the context of a word $j$ can be rewritten as
$$
p(i\mid j)\approx p_i\cdot e^{\mathbf{x}_j^\top\mathbf{x}_i-\mathbf{y}_j^\top\mathbf{y}_i}.
$$
Interestingly, embeddings of the first type ($\mathbf{x}_i$ and $\mathbf{x}_j$) are responsible for pulling the word $i$ into the context of the word $j$, while embeddings of the second type ($\mathbf{y}_i$ and $\mathbf{y}_j$) are responsible for pushing the word $i$ away from the context of the word $j$. We hypothesize that the $\mathbf{x}$-embeddings are more related to semantics, whereas the $\mathbf{y}$-embeddings are more related to syntax. We defer testing of this hypothesis to our future work.
\end{remark}

\section{Empirical verification}
\begin{figure*}[h]
\begin{center}
\includegraphics[width=.45\textwidth]{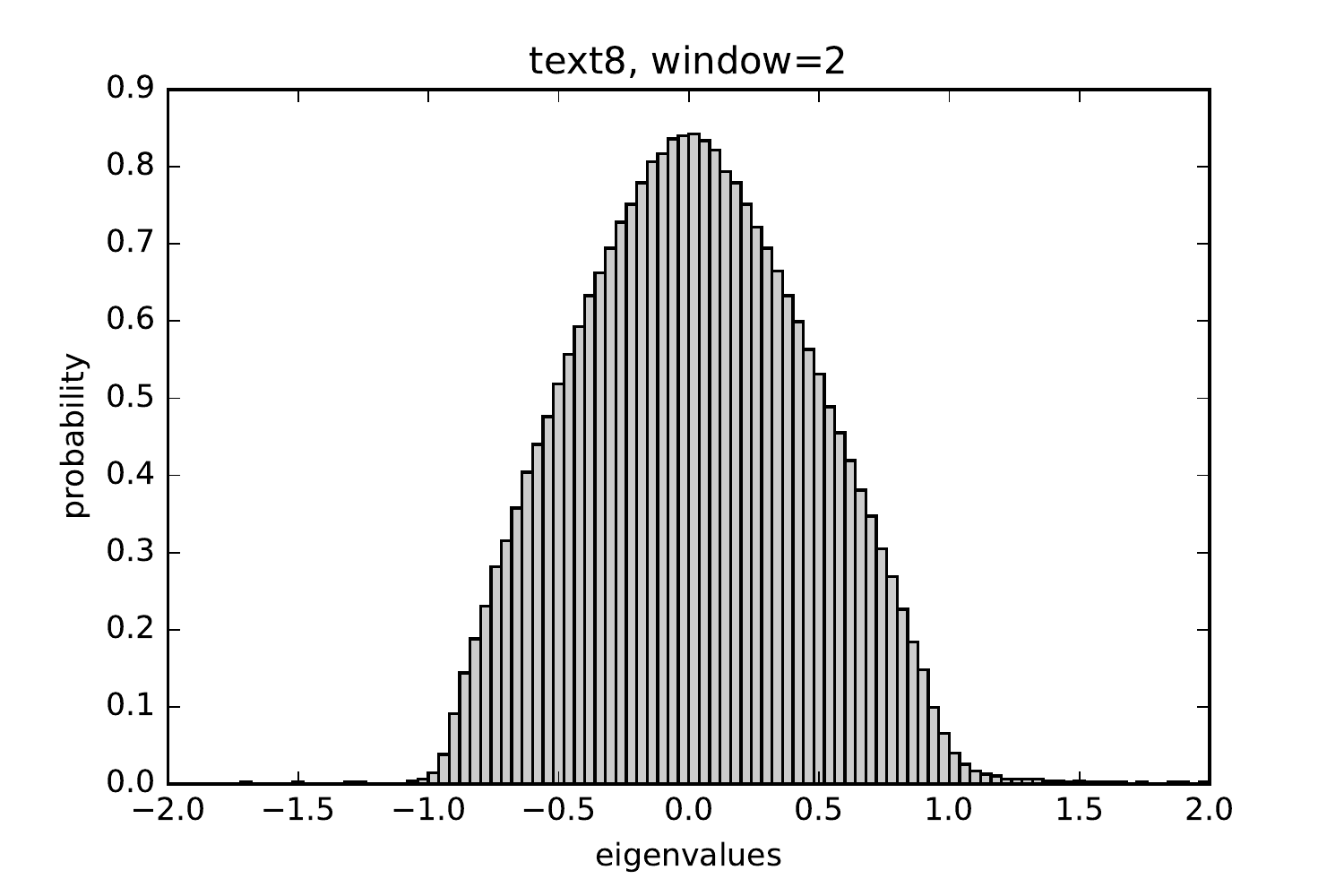}
\includegraphics[width=.45\textwidth]{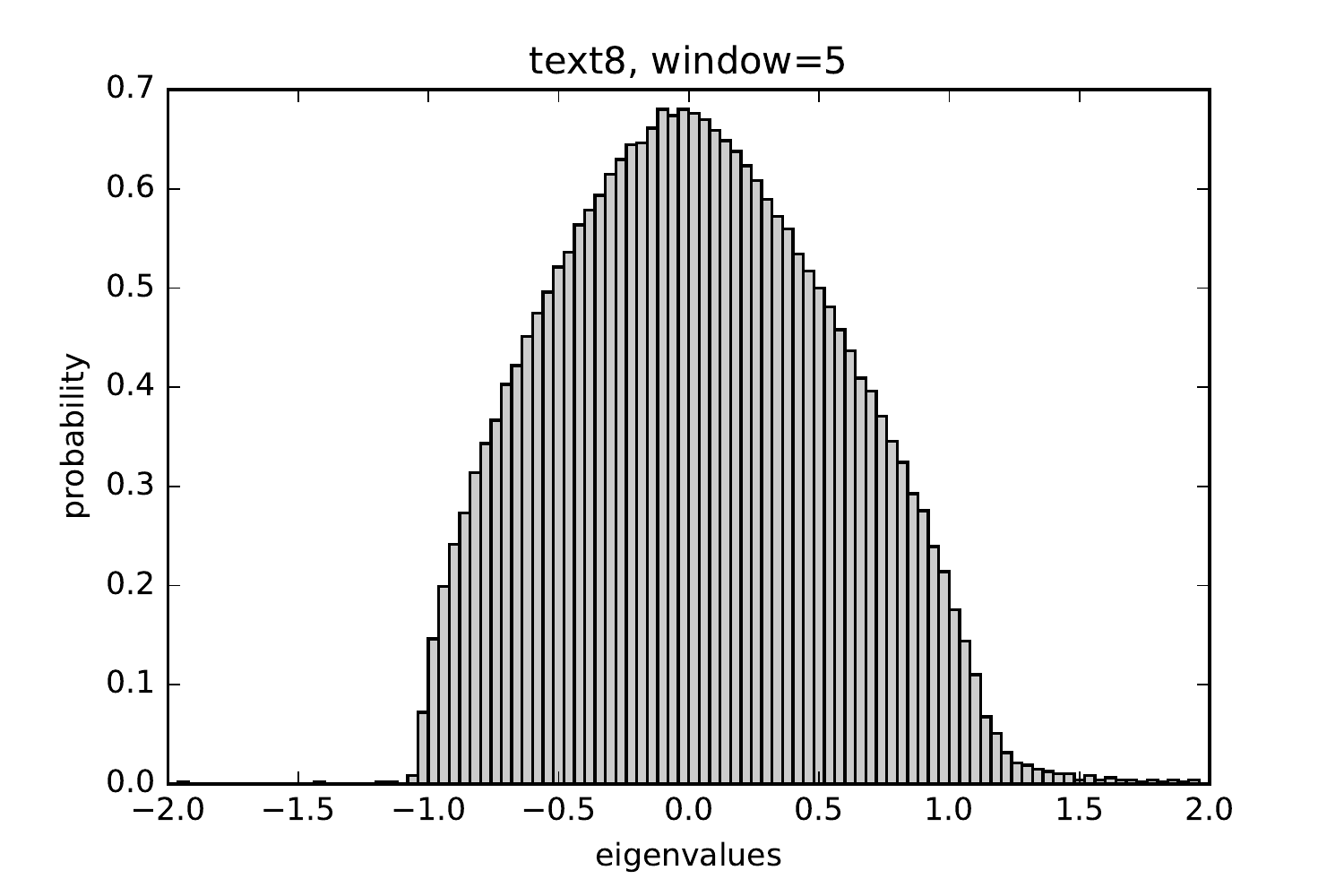}\\
\includegraphics[width=.45\textwidth]{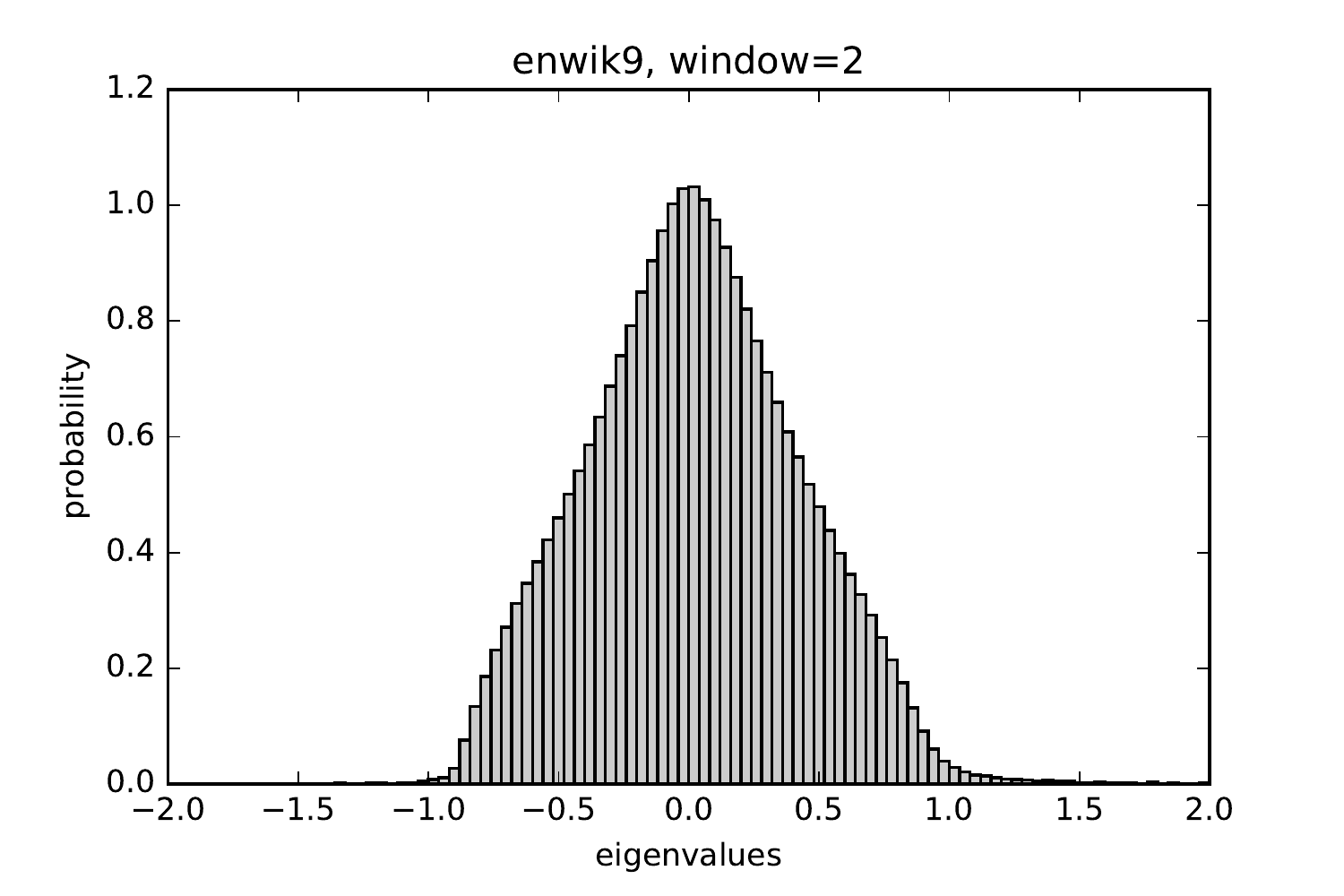}
\includegraphics[width=.45\textwidth]{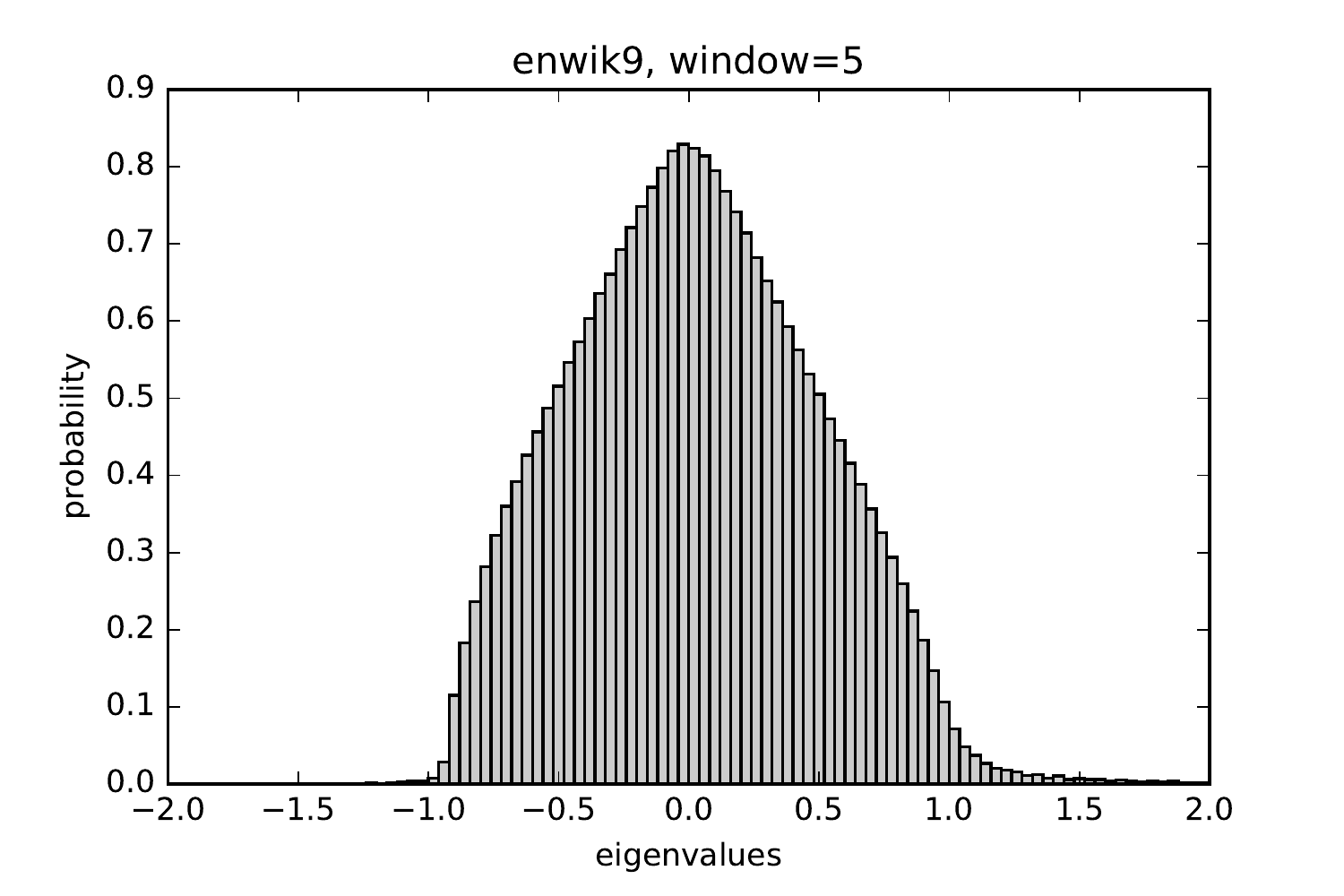}
\end{center}
\caption{Empirical distribution of eigenvalues of PMI matrices.}
\label{fig:eigvals}
\end{figure*}
To verify that the real-world PMI matrices have indeed symmetric (around 0) distribution of their eigenvalues, we consider two widely-used datasets, text8 and enwik9,\footnote{\url{http://mattmahoney.net/dc/textdata.html}. The enwik9 data was processed with the Perl-script \texttt{wikifil.pl} provided on the same webpage. It filters Wikipedia XML dumps to ``clean'' text consisting only of lowercase letters and spaces (never consecutive).} from which we extract PMI matrices using the \texttt{hyperwords} tool of \shortciteA{levy2015improving}. We use the default settings for all hyperparameters, except word frequency threshold and context window size. We were ignoring words that appeared less than 100 times and 150 times in text8 and enwik9 correspondingly, resulting in vocabularies of 11,815 and 29,145 correspondingly. We additionally experiment with context window 5, which by default is set to 2, and which we believe could affect the results. The eigenvalues of the PMI matrices are then calculated using the TensorFlow library \shortcite{abadi2016tensorflow}, and the above-mentioned threshold of 150 for enwik9 was chosen to fit the resulting PMI matrix into the GPU memory (12GB, NVIDIA Titan X Maxwell). The histograms of eigenvalues are provided in Figure~\ref{fig:eigvals}.
\begin{table}
\begin{center}
\begin{tabular}{l r r r}
\toprule
Data set & Size & \multicolumn{1}{c}{$T$} & \multicolumn{1}{c}{$|\mathcal{W}|$} \\
 \midrule
text8 & 100 MB & 17M & 254K\\
enwik9 & 715 MB & 124M & 833K\\
\bottomrule
\end{tabular}
\end{center}
\vspace{-10pt}\caption{Corpus statistics. $T=$ total length in tokens; $|\mathcal{W}|=$ number of unique words.}
\label{corpus_stats}
\end{table}
As we can see, the distributions are not perfectly symmetric with a little right skewness, but in general they seem to be symmetric. Notice, that this is in stark contrast with the equation (2.5) from \shortciteA{arora2016latent}, which claims that the PMI matrix should be approximately positive semi-definite, i.e. that it should have mostly positive eigenvalues. Also, notice that the shapes of distributions are far from resembling the Wigner semicircle law $x\mapsto\frac{1}{2\pi}\sqrt{4-x^2}$, which is the limiting distribution for the eigenvalues of many random symmetric matrices with i.i.d. entries \shortcite{wigner1955characteristic,wigner1958distribution}. This means that the entries of a typical PMI matrix \textit{are} dependent, otherwise we would observe approximately semicircle distributions for its eigenvalues. Interestingly, there is a striking similarity between the shapes of distributions in Figure \ref{fig:eigvals} and of spectral densities of the scale-free random graphs \shortcite{farkas2001spectra} and random graphs with expected degrees \shortcite{preciado2017moment} which arise in physics and network science. Notice that the connection between human language structure and scale-free random graphs was observed previously by \shortciteA{cancho2001small}, and it would be interesting to dive deeper in this direction.

\section{Weight tying in skip-gram model}
\setlength{\tabcolsep}{3.5pt}
\begin{table}[h]
\centering
\begin{footnotesize}
\begin{tabular}{l | l c | c c c c | c c}
\toprule
\multirow{2}{*}{Data} & \multirow{2}{*}{Model} & \multirow{2}{*}{Size} & \shortciteauthor{finkelstein2002placing} & \shortciteauthor{bruni2012distributional} & \shortciteauthor{radinsky2011word} & \shortciteauthor{luong2013better} & Google & MSR \\
 & & & WordSim & MEN  & M. Turk & Rare Words &      & \\
 \midrule
\multirow{2}{*}{text8} & SGNS & 28M & .681    & .241 & .631    & .072       & .307 & .286 \\
 & SGNS+WT & 14M & .637    & .215 & .624    & .057       & .314 & .319 \\
\midrule
\multirow{2}{*}{enwik9} & SGNS & 87M & .671 & .268 & .662 & .213 & .558 & .410 \\
 & SGNS+WT & 44M & .633 & .236 & .639 & .175 & .516 & .429\\
 \bottomrule
\end{tabular}
\end{footnotesize}
\caption{Evaluation of word embeddings in the analogy tasks (Google and MSR) and in the similarity tasks (the rest). For word similarities evaluation metric is the Spearman's correlation with the human ratings, while for word analogies it is the percentage of correct answers. Model sizes are in number of trainable parameters.}
\label{tab:emb_eval}
\end{table}
We would like to apply our results to tie embeddings in the skip-gram model of \shortciteA{mikolov2013distributed} in a theoretically grounded way. One may argue that our key Assumption~\ref{as:model} differs from the softmax-prediction of the skip-gram model. Although this is true, in fact the softmax normalization is never used in practice when training skip-gram. Instead it is common to replace the softmax cross-entropy by the negative sampling objective (equation (4) in \shortciteA{mikolov2013distributed}), and its optimization is almost equivalent to finding a low-rank approximation of the shifted word-word PMI matrix in the form $\mathbf{w}_i^\top\mathbf{c}_j\approx\PMI_{ij}-\log k$ \shortcite{levy2014neural}. Since our Assumptions lead to the same conclusion up to a constant shift \eqref{eq:pmi_factorization}, we believe that Theorem~\ref{th:main} can be directly applied to tie word ($\mathbf{w}_i$) and context ($\mathbf{c}_i$) embeddings in the SGNS model. For this purpose we form a vector $\mathbf{q}\in\mathbb{R}^{d}$ of $d$ i.i.d. draws from the Rademacher distribution\footnote{Rademacher distribution is a discrete probability distribution where a random variate $X$ has a 50\% chance of being $+1$ and a 50\% chance of being $-1$.} and then put 
\begin{equation}
\mathbf{c}_i=\mathbf{q} \odot \mathbf{w}_i\label{eq:wt}
\end{equation}
for all words $i$ in the vocabulary. This is equivalent to \eqref{eq:cont_vec} when the matrix $\mathbf{Q}$ has the special diagonal form \eqref{eq:q_diag}. Such modification of the SGNS is refered to as `SGNS + WT'. The word embeddings $\mathbf{w}_i$ are initialized randomly, and then trained on text8 and enwik9 using the reference word2vec implementation from the TensorFlow codebase\footnote{\url{https://github.com/tensorflow/models/blob/master/tutorials/embedding/word2vec.py}} with all hyperparameters set to their default values except that we choose the learning rate to decay 20\% faster in the weight-tied model. This additional tuning of the learning rate decay is not surprising: the model with tied embeddings has two times less parameters compared to the model with untied weights, and this leads to a significant change of the optimization landscape, which in turn results in the need to tune the most sensitive hyperparameter --- the learning rate (or its decay schedule). As is standard nowadays the trained embeddings are evaluated on several word similarity and word analogy tasks. We used the hyperwords tool of \shortciteA{goldberg2014word2vec} and we refer the reader to their paper for the methodology of evaluation. We only mention here few key points:
\begin{itemize}
    \item Our goal is not to beat state of the art, but to empirically validate the statement of Theorem~\ref{th:main}. This is why we were evaluating only word (input) embeddings for both SGNS and SGNS+WT. I.e., we were not adding context vectors to word vectors in the similarity tasks, as it is usually done nowadays.
    \item Word similarity datasets contain word pairs together with human-assigned similarity scores. The word vectors are evaluated by ranking the pairs according to their cosine similarities and measuring the correlation (Spearman's $\rho$) with the human ratings.
    \item For answering analogy questions ($a$ is to $b$ as $c$ is to $?$) we use the 3CosMul of \shortciteA{levy2014linguistic} and the evaluation metric for the analogy questions is the percentage of correct answers.
\end{itemize}  
The results of evaluation are provided in Table \ref{tab:emb_eval}. As we can see, SGNS + WT produces embeddings comparable in quality with those produced by the baseline SGNS model despite having 50\% fewer parameters. This also empirically validates the statement of our Theorem~\ref{th:main}. We notice that similar results can be obtained by letting the linear transform $\mathbf{Q}$ be a trainable matrix as shown by \shortciteA{gulordava2018represent}. The main difference of our approach is that we know exactly the form of $\mathbf{Q}$, and thus we do not need to learn it.

\section{Conclusion}
There is a remarkable relationship between human language and other branches of science, and we can get interesting and practical results by studying deeper such relationships. For example, the modern theory of random matrices is replete with theoretical results that can be immediately applied to models of natural language once such models are cast into the appropriate probabilistic setting, as is done in this paper.

\vskip 0.2in
\bibliography{sample}
\bibliographystyle{theapa}

\end{document}